\documentclass[11pt]{amsart}

\usepackage[T1]{fontenc}
\usepackage[utf8]{inputenc}
\usepackage{lmodern}
\usepackage[a4paper,margin=2.7cm]{geometry}
\usepackage{amsmath,amssymb,amsthm,mathtools}
\usepackage{mathrsfs}
\usepackage{enumitem}
\usepackage{hyperref}
\usepackage{xcolor}
\usepackage{setspace}
\usepackage[authoryear,round,sort]{natbib}
\hypersetup{
  colorlinks=true,
  linkcolor=blue!50!black,
  citecolor=blue!50!black,
  urlcolor=blue!50!black
}

\theoremstyle{definition}
\newtheorem{definition}{Definition}[section]
\newtheorem{remark}[definition]{Remark}

\theoremstyle{plain}
\newtheorem{proposition}[definition]{Proposition}

\numberwithin{equation}{section}
\title[Non reciprocal PCs with noise]{Noisy Nonreciprocal Pairwise Comparisons:
Scale Variation, Noise Calibration, and Admissible Ranking Regions}

\author{Jean-Pierre Magnot}

\begin{document}
\onehalfspacing

\maketitle

\begin{center}

{SFR MATHSTIC, LAREMA, Universit\'e d’Angers, 

2 Bd Lavoisier, 
49045 Angers cedex 1, France}

{Lyc\'ee Jeanne d'Arc, 40 avenue de Grande Bretagne, 63000 Clermont-Ferrand, 
France}

{Lepage Research Institute, 17 novembra 1, 081 16 Presov, Slovakia}

{magnot@math.cnrs.fr, jean-pierr.magnot@ac-clermont.fr}
\end{center}
\date{}

\begin{abstract}
Pairwise comparisons are widely used in decision analysis, preference modeling, and evaluation problems. In many practical situations, the observed comparison matrix is not reciprocal. This lack of reciprocity is often treated as a defect to be corrected immediately. In this article, we adopt a different point of view: part of the nonreciprocity may reflect a genuine variation in the evaluation scale, while another part is due to random perturbations.

We introduce an additive model in which the unknown underlying comparison matrix is consistent but not necessarily reciprocal. The reciprocal component carries the global ranking information, whereas the symmetric component describes possible scale variation. Around this structured matrix, we add a random perturbation and show how to estimate the noise level, assess whether the scale variation remains moderate, and assign probabilities to admissible ranking regions in the sense of strict ranking by pairwise comparisons. We also compare this approach with the brutal projection onto reciprocal matrices, which suppresses all symmetric information at once.

The Gaussian perturbation model is used here not because human decisions are exactly Gaussian, but because observed judgment errors often result from the accumulation of many small effects. In such a context, the central limit principle provides a natural heuristic justification for Gaussian noise. This makes it possible to derive explicit estimators and probability assessments while keeping the model interpretable for decision problems.
\end{abstract}



\section*{Introduction}

Pairwise comparisons are among the most classical tools in decision analysis, preference elicitation, and multicriteria evaluation. Their appeal is easy to understand: in many practical situations, decision makers find it easier to compare two alternatives at a time than to provide direct global scores for an entire set of options. For this reason, pairwise comparison methods appear in a wide range of contexts, from psychology and preference measurement to operational research, group decision making, and multicriteria synthesis.

In the ideal reciprocal setting, the comparison of alternative $i$ against alternative $j$ determines the opposite comparison automatically. In an additive formulation, reciprocity means that the comparison matrix $A=(a_{ij})$ satisfies
\(
a_{ji}=-a_{ij}.
\)
When, in addition, the comparisons satisfy the triangular relation
\[
a_{ik}=a_{ij}+a_{jk},
\]
the matrix is additively consistent and can be represented by a latent score vector $u=(u_1,\dots,u_n)$ through
\[
a_{ij}=u_i-u_j.
\]
This classical picture is appealing because it yields a clean global ranking structure and a direct interpretation of pairwise judgments.

Observed comparison data, however, are often less regular. In real applications, one frequently encounters matrices that are not reciprocal, not exactly consistent, or both. Such deviations may come from several sources: random judgment errors, hesitation, context dependence, heterogeneous elicitation conditions, aggregation of individual assessments, or changes in the effective evaluation scale. In most practical workflows, nonreciprocity is treated as an imperfection that should be corrected before further analysis. A typical reaction is to project the observed matrix onto the reciprocal matrices and then to extract a ranking from the antisymmetric part.

This standard reaction is natural and often useful, but it may also be too crude. Indeed, once a comparison matrix is observed to be nonreciprocal, two very different interpretations become possible. On the one hand, nonreciprocity may simply reflect noise. On the other hand, it may reveal a structured deformation of the evaluation process. In such a case, opposite comparisons are not exact reversals of one another, not because the judgments are purely random, but because the effective evaluation scale has shifted in a systematic way. From a decision-support viewpoint, these two interpretations should not be conflated.

The starting point of the present paper is therefore the following question: when a pairwise comparison matrix is nonreciprocal, how much of this asymmetry should be interpreted as residual noise, and how much may correspond to a meaningful but moderate deformation of the evaluation scale?

To address this question, we consider an additive model in which the observed matrix is written as
\[
x_{ij}=u_i-u_j+s_i+s_j+\varepsilon_{ij},
\qquad i\neq j.
\]
The term $u_i-u_j$ carries the latent ranking information, the term $s_i+s_j$ describes a structured scale deformation, and the residual term $\varepsilon_{ij}$ captures random perturbation. In this framework, the underlying structured matrix is generally nonreciprocal and should not be described as additively consistent in the classical sense. The relevant issue is rather whether the scale deformation remains moderate enough to preserve the latent ranking induced by $u$.

This viewpoint leads to a structured analysis with three levels.

First, we estimate separately the latent ranking component and the scale-deformation component by exploiting the antisymmetric and symmetric parts of the observed matrix. This separation is important because the symmetric part should not automatically be interpreted as pure error.

Second, we calibrate the residual perturbation after the structured part has been removed. To do so, we introduce residual reciprocity and residual triangular indicators. Under a Gaussian working model, these indicators provide explicit estimators of the effective residual noise level and of the dependence between opposite residual judgments. The Gaussian assumption is not meant as a literal description of human decision making. Rather, it is used as an interpretable first-order approximation: when an observed error is the cumulative result of many small and heterogeneous effects, the central limit principle suggests that a Gaussian approximation is often a reasonable working hypothesis.

Third, we use this residual calibration to assign probabilities to ranking regions. In other words, rather than outputting only one central ranking, we associate probabilities with neighboring strict rankings of the alternatives. This is particularly useful when the latent scores are close to each other and the decision maker needs not only an ordering, but also an assessment of how strongly the data support that ordering.

An important consequence of the proposed framework is that the structured approach and brutal reciprocal projection do not primarily differ at the level of the central ranking estimate. In the least-squares setting considered here, they yield the same estimator of the latent ranking vector. Their main difference lies instead in the interpretation of the symmetric part and therefore in the calibration of uncertainty. When the fitted scale deformation is non-negligible, brutal reciprocal projection may treat a structured phenomenon as if it were random noise, thereby producing an overly diffuse uncertainty assessment on the space of rankings.

The contribution of the paper may therefore be summarized as follows:
\begin{enumerate}
\item we introduce a scale-deformed additive model for noisy nonreciprocal pairwise comparisons;
\item we provide explicit estimators of the latent ranking vector and of the scale-deformation vector;
\item we calibrate the effective residual noise level through residual reciprocity and residual triangular indicators;
\item we define decision-oriented diagnostics quantifying whether the fitted scale deformation is globally weak or locally influential;
\item we assign probabilities to ranking regions and compare this probabilistic output with that obtained under brutal reciprocal projection.
\end{enumerate}

The rest of the paper is organized as follows. Section~2 clarifies the conceptual distinction between classical additive consistency and structured scale deformation, and introduces the basic model. Section~3 estimates the latent ranking and scale-deformation components. Section~4 calibrates the residual noise and introduces quantitative diagnostics of moderate scale deformation. Section~5 defines ranking-region probabilities and compares the proposed framework with brutal reciprocal projection. The final section presents numerical illustrations and discusses practical implications for decision analysis.

\subsection*{Review of related works}

The literature relevant to the present paper spans several traditions.

A first family of references concerns the classical foundations of pairwise comparisons, comparative judgment, and latent-score models. The psychological roots of the area go back to Thurstone's law of comparative judgment \cite{Thurstone1927}, while statistical paired-comparison models were formalized in the Bradley--Terry framework \cite{BradleyTerry1952}. Luce's axiomatic treatment of individual choice behavior \cite{Luce1959} further shaped the conceptual background of preference modeling. In operational research and decision analysis, the analytic hierarchy process of Saaty \cite{Saaty1977,Saaty1990} provided one of the most influential frameworks for using pairwise comparisons in practice, and broad overviews of its later developments are given in \cite{IshizakaLabib2011}.

A second family of works studies inconsistency, perturbations, and sensitivity issues in pairwise comparison matrices. Important contributions include Koczkodaj's inconsistency viewpoint \cite{Koczkodaj1993}, Forman's analysis of random indices for incomplete matrices \cite{Forman1990}, and the perturbation-based study of Farkas and R\'ozsa \cite{FarkasRozsa2001}. Statistical approaches to consistency have also been proposed, for instance in \cite{AlonsoLamata2005}, while empirical analyses of inconsistency patterns can be found in \cite{BozokiEtAl2013}. A useful broader survey of inconsistency indices is provided by Brunelli \cite{Brunelli2018}. More recently, concerns have also been raised about the reliability of pairwise-comparison-based decision procedures in certain settings \cite{TriantaphyllouYanase2024}.

A third line of work deals specifically with nonreciprocity. Experimental evidence on the practical relevance of nonreciprocity in AHP-type elicitation was discussed by Linares, Lumbreras, Santamar\'{\i}a, and Veiga \cite{LinaresEtAl2016}. Methodological and algebraic approaches to nonreciprocal pairwise comparisons appear in Nishizawa \cite{Nishizawa2019}, Liu, Liu, and Hu \cite{LiuLiuHu2023}, Hu, Liu, and Cai \cite{HuLiuCai2024}, and Mazurek and Linares \cite{MazurekLinares2024}. The present article belongs to this stream, but differs from it by focusing explicitly on the distinction between structured scale deformation and residual noise.

A fourth family of references concerns fuzzy, incomplete, or imprecise preference relations. Buckley's fuzzy hierarchical analysis \cite{Buckley1985} is a classical starting point in this direction. Multiplicative preference relations and group decision procedures were developed by Herrera, Herrera-Viedma, and Chiclana \cite{HerreraViedmaEtAl2001}, while several important works addressed consistency and incompleteness issues in fuzzy preference relations \cite{HerreraViedmaEtAl2004,HerreraViedmaEtAl2007}. The ordinal-consistency viewpoint for fuzzy preference relations was further discussed in \cite{XuPatnayakuniWang2013}. These references are particularly relevant here because they show that uncertainty, incompleteness, and structured deviation from ideal reciprocity have long been recognized as central issues in decision-support models.

A fifth family of works studies stochastic, noisy, or probabilistic formulations of pairwise comparisons. The stochastic treatment of judgments in AHP and the probability of rank reversal were explored by Stam and Duarte Silva \cite{StamSilva1997}. Related stochastic group-preference modeling ideas appear in van den Honert \cite{VanDenHonert1998}. Bayesian and statistical approaches to decision making with noisy pairwise comparisons can be found in Hughes \cite{Hughes2009}, while stochastic multiple-criteria models based on random pairwise evaluations were developed in \cite{FanLiuFeng2010}. Additional stochastic AHP methodologies for imprecise preferences and stochastic preference analysis were proposed in \cite{JalaoWuShunk2014,DurbachEtAl2014,ZhuXu2014}. The present paper is close in spirit to this probabilistic tradition, but it places special emphasis on the interpretation of nonreciprocity as a possible mixture of structured deformation and residual perturbation.

A sixth group of works is closer to the present author's own trajectory. Group-valued and geometric structures on pairwise comparisons were developed in \cite{Magnot2019Heliyon}. A geometric study of random pairwise comparison matrices and reciprocal projection was proposed in \cite{Magnot2024JAA}. The present paper builds directly on these ideas, but shifts the focus toward decision-oriented uncertainty calibration. The strict-ranking viewpoint used here is also connected with \cite{Magnot2026AMS}, where ranking regions are studied more directly. Finally, the broader epistemological and measurement-oriented perspective of \cite{EllingsenLundholmMagnot2024} is relevant to the present work insofar as it emphasizes the plurality of measurement features and the need for interpretation beyond a purely formal correction of data.

Finally, it is worth mentioning recent attempts to reinterpret classical pairwise-comparison frameworks from the viewpoint of uncertainty itself. An uncertainty-induced axiomatic reading of AHP was proposed in \cite{LiuQiuZhang2021}, while critical discussion of that interpretation appeared in \cite{WangDeng2023}. These works are particularly close in spirit to the present article, since they illustrate the fact that deviations from ideal comparison structures may be treated not only as flaws, but also as meaningful signals requiring a refined interpretation.

Against this background, the present paper occupies an intermediate position between the literature on nonreciprocal pairwise comparison matrices, the literature on stochastic and noisy preference modeling, and the literature on geometric or structural treatments of ranking uncertainty. Its specific aim is to show that, even when the central ranking estimate remains unchanged, the interpretation of nonreciprocity may substantially affect the way uncertainty is quantified and communicated in decision problems.
\section{Position of the problem}

\subsection{Why nonreciprocity should not always be discarded}

Pairwise comparison matrices are classical tools in decision-making, preference elicitation, and multicriteria evaluation. They are especially useful when decision makers find it easier to compare alternatives two by two than to assign direct global scores.

In an ideal reciprocal setting, comparing item $i$ to item $j$ determines the opposite comparison automatically. In an additive framework, reciprocity reads
\(
a_{ji}=-a_{ij}.
\)
Such a property is natural when the elicitation scale is perfectly stable and opposite comparisons are exact reversals of one another.

In practice, however, observed pairwise comparison matrices are often nonreciprocal. This phenomenon is usually treated as a defect that should be corrected immediately. In many applications, this is a reasonable first reaction. Yet such a treatment may also be too crude. Indeed, nonreciprocity may have at least two different origins:
\begin{enumerate}[label=\textup{(\arabic*)}]
\item a \emph{random origin}, due to hesitation, fatigue, context effects, limited attention, or local judgment instability;
\item a \emph{structured origin}, due to a moderate variation in the effective evaluation scale.
\end{enumerate}

The aim of this article is to separate these two effects. Rather than suppressing nonreciprocity blindly, we first ask whether part of it may reflect a structured and interpretable deformation of the evaluation process.

\subsection{Classical additive consistency and reciprocity}

Before introducing the model used in this paper, it is important to distinguish two notions that should not be confused.

In the classical additive framework, a pairwise comparison matrix $A=(a_{ij})$ is said to be \emph{additively consistent} if
\begin{equation}\label{eq:addconsistency}
a_{ik}=a_{ij}+a_{jk}
\qquad
\text{for all pairwise distinct } i,j,k.
\end{equation}
Under the usual convention $a_{ii}=0$, this property forces the matrix to be of the form
\[
a_{ij}=u_i-u_j
\]
for some score vector $u=(u_1,\dots,u_n)\in\mathbb R^n$. In particular, additive consistency implies reciprocity.
Thus, a genuinely nonreciprocal comparison matrix should \emph{not} be called additively consistent in the classical sense.

\begin{proposition}
Let $A=(a_{ij})_{1\le i,j\le n}$ be a real matrix with $a_{ii}=0$ for all $i$. Assume that
\[
a_{ik}=a_{ij}+a_{jk}
\qquad
\text{for all pairwise distinct } i,j,k.
\]
Then there exists a vector $u=(u_1,\dots,u_n)\in\mathbb R^n$ such that
\[
a_{ij}=u_i-u_j
\qquad
\text{for all } i,j.
\]
In particular, $A$ is reciprocal.
\end{proposition}

\begin{proof}
Fix an index $r$ and define
\[
u_i:=a_{ir}
\qquad (1\le i\le n).
\]
Then $u_r=a_{rr}=0$. Let $i,j$ be distinct. By additive consistency,
\(
a_{ij}=a_{ir}+a_{rj}.
\)
Applying the same relation to the triple $(r,j,r)$ gives
\[
a_{rr}=a_{rj}+a_{jr},
\Rightarrow
a_{rj}=-a_{jr}.
\Rightarrow
a_{ij}=a_{ir}-a_{jr}=u_i-u_j.
\]
This proves the claim, and reciprocity follows immediately.
\end{proof}

\subsection{Latent ranking and scale deformation}

The present paper is based on a different viewpoint. We assume that the underlying decision structure is driven by a latent ranking vector
\[
u=(u_1,\dots,u_n),
\qquad
\sum_{i=1}^n u_i=0,
\]
but that the elicitation process may also involve a systematic variation of the effective evaluation scale, described by a centered vector
\[
s=(s_1,\dots,s_n),
\qquad
\sum_{i=1}^n s_i=0.
\]

This leads to the structured matrix
\begin{equation}\label{eq:scaledeformed}
m_{ij}^{\mathrm{sd}}:=u_i-u_j+s_i+s_j,
\qquad i\neq j,
\end{equation}
which we call a \emph{scale-deformed comparison matrix}.

The meaning of \eqref{eq:scaledeformed} is simple:
\begin{itemize}
\item the term $u_i-u_j$ carries the latent ranking information;
\item the term $s_i+s_j$ describes a systematic deformation of the evaluation scale.
\end{itemize}

In general, the matrix $M^{\mathrm{sd}}=(m_{ij}^{\mathrm{sd}})$ is neither reciprocal nor additively consistent. This is not a defect of the model, but part of its intended interpretation: nonreciprocity is allowed as a structured feature of the elicitation process.

\subsection{Ranking compatibility}

The key decision question is not whether the scale-deformed matrix is additively consistent, but whether the scale deformation remains moderate enough to preserve the latent ranking induced by $u$.

\begin{definition}
The scale-deformed comparison matrix $M^{\mathrm{sd}}$ is said to be \emph{ranking-compatible} with the latent ranking vector $u$ if
\(
\operatorname{sign}(m_{ij}^{\mathrm{sd}})
=
\operatorname{sign}(u_i-u_j)
\qquad
\text{for all } i\neq j.
\)
\end{definition}

In other words, the structured nonreciprocity does not reverse any pairwise preference at the latent ranking level.

A simple sufficient condition is obtained in terms of the smallest latent ranking gap. Let
\(
u_{(1)}>u_{(2)}>\cdots>u_{(n)}
\)
be the decreasing rearrangement of the coordinates of $u$, and define
\[
\mathrm{gap}(u):=\min_{1\le k\le n-1}\bigl(u_{(k)}-u_{(k+1)}\bigr).
\]

\begin{proposition}
If
\(
2\|s\|_\infty<\mathrm{gap}(u),
\)
then the scale-deformed matrix $M^{\mathrm{sd}}$ is ranking-compatible with $u$.
\end{proposition}

\begin{proof}
For all $i\neq j$,
\(
|s_i+s_j|\le 2\|s\|_\infty.
\)
If $u_i>u_j$, then
\[
u_i-u_j\ge \mathrm{gap}(u).
\Rightarrow
m_{ij}^{\mathrm{sd}}
=
(u_i-u_j)+(s_i+s_j)
\ge
\mathrm{gap}(u)-2\|s\|_\infty
>
0.
\]
Thus the sign of $m_{ij}^{\mathrm{sd}}$ agrees with the sign of $u_i-u_j$ for all $i\neq j$.
\end{proof}

\begin{remark}
This sufficient condition expresses in a transparent way what we mean by \emph{moderate scale variation}: the deformation is allowed, but it remains too small to overturn the latent ranking.
\end{remark}

\subsection{The observed matrix and the noise model}

The structured matrix $M^{\mathrm{sd}}$ is not observed directly. Instead, we observe
\(
X=(x_{ij})_{1\le i,j\le n}
\)
with
\begin{equation}\label{eq:observedmodel}
x_{ij}=u_i-u_j+s_i+s_j+\varepsilon_{ij},
\qquad i\neq j.
\end{equation}
The random term $\varepsilon_{ij}$ represents perturbations of the elicited judgment.

At this point, one may ask why a Gaussian model is relevant. Decision makers do not produce literally Gaussian judgments. However, an observed comparison error often results from the accumulation of many small effects: attention fluctuations, hesitations, memory effects, contextual biases, or imperfections of the elicitation device. When many small influences combine additively, the central limit principle suggests that the overall perturbation is often reasonably approximated by a Gaussian law.

For this reason, we use a Gaussian perturbation model as a tractable working framework. Its role is not to claim exact normality, but to provide explicit formulas and interpretable probability statements.

\subsection{What has to be decided from the data}

Starting from the observed matrix $X$, our objective is not merely to output one ranking. We want to answer four decision-oriented questions.

\medskip

\noindent\textbf{Question 1. Noise calibration.}
Can we estimate the standard deviation of the residual perturbation? This gives a quantitative measure of judgment reliability.

\medskip

\noindent\textbf{Question 2. Moderate scale variation.}
Can we determine whether the structured scale deformation is small enough not to alter the latent ranking in a practically significant way?

\medskip

\noindent\textbf{Question 3. Ranking-region probabilities.}
Once the noise level has been estimated, can we assign probabilities to the admissible ranking regions associated with the latent ranking scores?

\medskip

\noindent\textbf{Question 4. Comparison with brutal reciprocal projection.}
How different is the present structured approach from the simpler strategy that first projects the observed matrix onto the reciprocal matrices and then ranks the items?

\subsection{A first decomposition of the observed matrix}

The observed matrix splits naturally into its antisymmetric and symmetric parts:
\[
K:=\frac{X-X^\top}{2},
\qquad
H:=\frac{X+X^\top}{2}.
\]
Using \eqref{eq:observedmodel}, we obtain
\[
K_{ij}=u_i-u_j+\frac{\varepsilon_{ij}-\varepsilon_{ji}}{2},
\qquad 
\hbox{and}
\qquad
H_{ij}=s_i+s_j+\frac{\varepsilon_{ij}+\varepsilon_{ji}}{2}.
\]

This decomposition has a direct interpretation:
\begin{itemize}
\item the antisymmetric part $K$ contains the latent ranking information;
\item the symmetric part $H$ contains the scale-deformation information.
\end{itemize}

Hence, the brutal reciprocal projection keeps only $K$ and discards $H$. Our point of view is different: before removing the symmetric part, one should determine whether it is negligible noise or a meaningful structured effect.

\subsection{Roadmap of the paper}

The next section introduces the estimation of the latent ranking component and of the scale-deformation component. We then study residual indicators allowing us to calibrate the effective noise level and to assess whether the scale variation remains moderate. After that, we assign probabilities to ranking regions and compare the resulting uncertainty quantification with that obtained under brutal reciprocal projection.

\section{Estimating the latent ranking and the scale deformation}

\subsection{Why a separate estimation step is needed}

The observed matrix $X$ mixes three effects:
\begin{enumerate}[label=\textup{(\arabic*)}]
\item the latent ranking component, carried by the differences $u_i-u_j$;
\item the structured scale-deformation component, carried by the sums $s_i+s_j$;
\item the residual perturbation component, carried by the noise terms $\varepsilon_{ij}$.
\end{enumerate}

If one projects $X$ directly onto the reciprocal matrices, then all symmetric information is removed at once. This may be appropriate when the symmetric part is negligible, but it may also erase a meaningful deformation of the evaluation scale. For this reason, we first estimate the two structured components separately and only then analyze the residual noise.

\subsection{Antisymmetric and symmetric parts}

Recall that
\[
K:=\frac{X-X^\top}{2},
\qquad
H:=\frac{X+X^\top}{2}.
\]
Under model \eqref{eq:observedmodel}, we have
\[
K_{ij}=u_i-u_j+\eta_{ij},
\qquad
\eta_{ij}:=\frac{\varepsilon_{ij}-\varepsilon_{ji}}{2},
\]
and
\[
H_{ij}=s_i+s_j+\zeta_{ij},
\qquad
\zeta_{ij}:=\frac{\varepsilon_{ij}+\varepsilon_{ji}}{2}.
\]

This suggests estimating the latent ranking from $K$ and the scale deformation from $H$.

\subsection{Estimating the latent ranking scores}

The antisymmetric matrix $K$ is close to a reciprocal additive matrix of the form
\(
(u_i-u_j)_{i,j}.
\)
A natural estimator is obtained by least squares:
\[
\min_{\sum_i u_i=0}
\sum_{i\neq j}\bigl(K_{ij}-(u_i-u_j)\bigr)^2.
\]

On the complete comparison graph, this minimization problem has an explicit solution.

\begin{proposition}
Let $K=(K_{ij})$ be the antisymmetric part of the observed matrix. The unique least-squares estimator of the vector $u=(u_1,\dots,u_n)$ under the constraint $\sum_i u_i=0$ is
\begin{equation}\label{eq:uhat2}
\widehat u_i=\frac1n\sum_{j=1}^n K_{ij},
\qquad 1\le i\le n.
\end{equation}
\end{proposition}

\begin{proof}
Consider the criterion
\[
\Phi(u):=\sum_{i\neq j}\bigl(K_{ij}-(u_i-u_j)\bigr)^2
\]
under the constraint $\sum_i u_i=0$. Differentiating with respect to $u_i$ and using the antisymmetry of $K$, one obtains
\[
n\,u_i=\sum_{j=1}^n K_{ij},
\]
which gives \eqref{eq:uhat2}. Uniqueness follows from strict convexity on the hyperplane $\sum_i u_i=0$.
\end{proof}

\begin{remark}
Formula \eqref{eq:uhat2} has a simple interpretation: the estimated latent score of item $i$ is the average signed advantage of $i$ against all the other items, computed from the antisymmetric part of the observed data.
\end{remark}

\subsection{Estimating the scale deformation}

We now turn to the symmetric part $H$. Under the model, $H$ is close to a matrix of the form
\(
(s_i+s_j)_{i,j}.
\)
We estimate $s=(s_1,\dots,s_n)$ by least squares:
\[
\min_{\sum_i s_i=0}
\sum_{i\neq j}\bigl(H_{ij}-(s_i+s_j)\bigr)^2.
\]

Again, on the complete graph, the solution is explicit.

\begin{proposition}
Let $H=(H_{ij})$ be the symmetric part of the observed matrix, and define
\[
r_i:=\sum_{j\neq i}H_{ij},
\qquad
\bar r:=\frac1n\sum_{i=1}^n r_i.
\]
Then the unique least-squares estimator of the vector $s=(s_1,\dots,s_n)$ under the constraint $\sum_i s_i=0$ is
\begin{equation}\label{eq:shat2}
\widehat s_i=\frac{r_i-\bar r}{n-2},
\qquad 1\le i\le n.
\end{equation}
\end{proposition}

\begin{proof}
Consider
\(
\Psi(s):=\sum_{i\neq j}\bigl(H_{ij}-(s_i+s_j)\bigr)^2
\)
under the constraint $\sum_i s_i=0$. Differentiating with respect to $s_i$ and using the symmetry of $H$, one obtains
\(
(n-2)s_i=r_i-\bar r,
\)
which gives \eqref{eq:shat2}. Uniqueness follows from strict convexity on the hyperplane $\sum_i s_i=0$.
\end{proof}

\begin{remark}
The quantity $\widehat s_i$ measures how much item $i$ contributes to a systematic local deformation of the evaluation scale. When all $\widehat s_i$ are close to zero, the reciprocal approximation is likely to be adequate. When some of them are substantial, suppressing the symmetric part may remove meaningful information.
\end{remark}

\subsection{Fitted structured matrix and residual matrix}

Once $\widehat u$ and $\widehat s$ have been computed, we define the fitted structured matrix and the residual matrix respectively by
\[
\widehat M_{ij}^{\mathrm{sd}}
:=
\widehat u_i-\widehat u_j+\widehat s_i+\widehat s_j,
\qquad i\neq j, \qquad
\hbox{ and } \qquad
\widehat E_{ij}:=x_{ij}-\widehat M_{ij}^{\mathrm{sd}}.
\]

The matrix $\widehat E$ is the part of the data left unexplained after extracting the latent ranking component and the scale-deformation component. It is this residual matrix, rather than the raw matrix $X$, that should be used to calibrate the effective noise level.
\subsection{What has to be decided from the data}

Starting from the observed matrix $X$, our objective is not merely to produce one ranking. We want to answer four decision-oriented questions.

\medskip

\noindent\textbf{Question 1. Noise calibration.}
Can we estimate the standard deviation $\sigma$ of the random perturbation? This provides a quantitative measure of the reliability of the elicited data.

\medskip

\noindent\textbf{Question 2. Reasonable scale variation.}
Can we determine whether the nonreciprocal structured component, represented by $s$, remains moderate enough not to distort the global evaluation excessively? In other words, when is scale variation present but still acceptable from a decision point of view?

\medskip

\noindent\textbf{Question 3. Probabilities of admissible ranking regions.}
Once the noise level has been estimated, can we assign probabilities to the admissible loci associated with strict rankings? This question connects the present work with the geometric approach developed for strict ranking by pairwise comparisons.

\medskip

\noindent\textbf{Question 4. Comparison with brutal reciprocal projection.}
How different is the proposed structured treatment from the simpler strategy that first projects the observed matrix onto the reciprocal matrices and then ranks the items? This question connects the present work with the geometric projection approach developed for random pairwise comparison matrices.

\subsection{A first structural decomposition}

A useful feature  is that the observed matrix splits naturally into its antisymmetric and symmetric parts:
\[
K:=\frac{X-X^\top}{2},
\qquad
H:=\frac{X+X^\top}{2}.
\]
Therefore,
\[
K_{ij}=u_i-u_j+\frac{\varepsilon_{ij}-\varepsilon_{ji}}{2},
\qquad 
\hbox{and}
\qquad
H_{ij}=s_i+s_j+\frac{\varepsilon_{ij}+\varepsilon_{ji}}{2}.
\]

This decomposition has a clear meaning:
\begin{itemize}
\item the antisymmetric part $K$ contains the ranking information;
\item the symmetric part $H$ contains the scale-variation information.
\end{itemize}

Hence the reciprocal projection approach keeps only $K$ and discards $H$. Our point of view is different: before suppressing the symmetric part, one should determine whether it represents negligible noise or a meaningful structured effect.

\section{Estimating the ranking and scale-variation components}

\subsection{Why a separate estimation step is needed}

The observed matrix $X$ mixes three different effects:
\begin{enumerate}[label=\textup{(\arabic*)}]
\item the global ranking component, carried by the differences $u_i-u_j$;
\item the structured scale-variation component, carried by the sums $s_i+s_j$;
\item the residual perturbation component, carried by the noise terms $\varepsilon_{ij}$.
\end{enumerate}

If one projects $X$ directly onto the reciprocal matrices, then all symmetric information is removed at once. This may be appropriate when the symmetric part is negligible, but it may also erase a meaningful variation of the evaluation scale. For this reason, we first estimate the two structured components separately and only then analyze the residual noise.

The model is
\begin{equation}\label{eq:mainmodel}
x_{ij}=u_i-u_j+s_i+s_j+\varepsilon_{ij},
\qquad i\neq j,
\end{equation}
with the identifiability conditions
\[
\sum_{i=1}^n u_i=0,
\qquad
\sum_{i=1}^n s_i=0.
\]

\subsection{Antisymmetric and symmetric decomposition}

As observed above, the matrix $X$ splits into its antisymmetric and symmetric parts:
\[
K:=\frac{X-X^\top}{2},
\qquad
H:=\frac{X+X^\top}{2}.
\]
Under model \eqref{eq:mainmodel}, we have
\[
K_{ij}=u_i-u_j+\eta_{ij},
\qquad
\eta_{ij}:=\frac{\varepsilon_{ij}-\varepsilon_{ji}}{2},
\]
and
\[
H_{ij}=s_i+s_j+\zeta_{ij},
\qquad
\zeta_{ij}:=\frac{\varepsilon_{ij}+\varepsilon_{ji}}{2}.
\]

This suggests estimating the ranking component from the antisymmetric part and the scale-variation component from the symmetric part.

\subsection{Estimating the ranking scores}

The antisymmetric matrix $K$ is close to an additive reciprocal matrix of the form
\(
(u_i-u_j)_{i,j}.
\)
A natural estimator is obtained by least squares:
\[
\min_{\sum_i u_i=0}
\sum_{i\neq j}\bigl(K_{ij}-(u_i-u_j)\bigr)^2.
\]

On the complete comparison graph, this minimization problem has an explicit solution.

\begin{proposition}
Let $K=(K_{ij})$ be the antisymmetric part of the observed matrix. The unique least-squares estimator of the vector $u=(u_1,\dots,u_n)$ under the constraint $\sum_i u_i=0$ is
\begin{equation}\label{eq:uhat}
\widehat u_i=\frac1n\sum_{j=1}^n K_{ij},
\qquad 1\le i\le n.
\end{equation}
\end{proposition}

\begin{proof}
Consider the criterion
\[
\Phi(u):=\sum_{i\neq j}\bigl(K_{ij}-(u_i-u_j)\bigr)^2
\]
under the constraint $\sum_i u_i=0$. Differentiating with respect to $u_i$ gives
\[
\frac{\partial \Phi}{\partial u_i}
=
-2\sum_{j\neq i}\bigl(K_{ij}-(u_i-u_j)\bigr)
+2\sum_{j\neq i}\bigl(K_{ji}-(u_j-u_i)\bigr).
\]
Since $K$ is antisymmetric, $K_{ji}=-K_{ij}$, hence
\[
\frac{\partial \Phi}{\partial u_i}
=
-4\sum_{j\neq i}K_{ij}
+4n\,u_i,
\]
because $\sum_j u_j=0$. Therefore the normal equations are
\[
n\,u_i=\sum_{j=1}^n K_{ij},
\]
which yields \eqref{eq:uhat}. Uniqueness follows from the strict convexity of $\Phi$ on the hyperplane $\sum_i u_i=0$.
\end{proof}

\begin{remark}
Formula \eqref{eq:uhat} has a simple interpretation: the estimated score of item $i$ is the average signed advantage of $i$ against all the other items, computed from the antisymmetric part of the observed matrix.
\end{remark}

\subsection{Estimating the scale-variation component}

We now turn to the symmetric part $H$. Under the model, $H$ is close to a matrix of the form
\(
(s_i+s_j)_{i,j}.
\)
We estimate $s=(s_1,\dots,s_n)$ by least squares:
\(
\min_{\sum_i s_i=0}
\sum_{i\neq j}\bigl(H_{ij}-(s_i+s_j)\bigr)^2.
\)
Again, on the complete graph, the solution is explicit.

\begin{proposition}
Let $H=(H_{ij})$ be the symmetric part of the observed matrix, and define
\[
r_i:=\sum_{j\neq i}H_{ij},
\qquad
\bar r:=\frac1n\sum_{i=1}^n r_i.
\]
Then the unique least-squares estimator of the vector $s=(s_1,\dots,s_n)$ under the constraint $\sum_i s_i=0$ is
\begin{equation}\label{eq:shat}
\widehat s_i=\frac{r_i-\bar r}{n-2},
\qquad 1\le i\le n.
\end{equation}
\end{proposition}

\begin{proof}
Consider
\[
\Psi(s):=\sum_{i\neq j}\bigl(H_{ij}-(s_i+s_j)\bigr)^2
\]
under the constraint $\sum_i s_i=0$. Differentiating with respect to $s_i$ yields
\[
\frac{\partial \Psi}{\partial s_i}
=
-2\sum_{j\neq i}\bigl(H_{ij}-(s_i+s_j)\bigr)
-2\sum_{j\neq i}\bigl(H_{ji}-(s_j+s_i)\bigr).
\]
Since $H$ is symmetric, $H_{ji}=H_{ij}$, hence
\[
\frac{\partial \Psi}{\partial s_i}
=
-4\sum_{j\neq i}H_{ij}
+4(n-2)s_i+4\sum_{k=1}^n s_k.
\]
Using the constraint $\sum_k s_k=0$, we obtain
\[
(n-2)s_i=r_i-\bar r,
\]
which gives \eqref{eq:shat}. Uniqueness again follows from strict convexity on the hyperplane $\sum_i s_i=0$.
\end{proof}

\begin{remark}
The quantity $\widehat s_i$ measures how much item $i$ contributes to a systematic local deformation of the evaluation scale. When all $\widehat s_i$ are close to $0$, the reciprocal approximation is likely to be adequate. When some of them are substantial, suppressing the symmetric part may remove information that is relevant for the interpretation of the judgments.
\end{remark}

\subsection{Fitted structured matrix and residual matrix}

Once $\widehat u$ and $\widehat s$ have been computed, we define the fitted structured matrix by
\[
\widehat M_{ij}:=\widehat u_i-\widehat u_j+\widehat s_i+\widehat s_j,
\qquad i\neq j,
\]
and the residual matrix by
\begin{equation}\label{eq:residual}
\widehat E_{ij}:=x_{ij}-\widehat M_{ij}.
\end{equation}

The matrix $\widehat E$ is the part of the data that remains unexplained after extracting the ranking component and the scale-variation component. It is this residual matrix, rather than the raw observed matrix, that should be used to estimate the effective noise level.

This distinction is important. Indeed, if one computes reciprocity or consistency indicators directly on $X$, the resulting quantities mix together:
\begin{itemize}
\item the global ranking structure,
\item the possible variation of the evaluation scale,
\item and the residual perturbation.
\end{itemize}
By contrast, indicators computed on $\widehat E$ are intended to capture the random part of the elicitation process.

\subsection{Ranking regions and decision uncertainty}

The vector $\widehat u=(\widehat u_1,\dots,\widehat u_n)$ provides a point estimate of the global ranking component. However, in decision problems, a single estimated ordering is often not enough. When the noise level is non-negligible, nearby score vectors may induce different strict orderings.

For this reason, it is useful to partition the score space into \emph{ranking regions}. Each ranking region corresponds to one strict ordering of the items. For example, the region associated with the order
\(
i_1 \succ i_2 \succ \cdots \succ i_n
\)
is the set of score vectors $v=(v_1,\dots,v_n)$ satisfying
\(
v_{i_1}>v_{i_2}>\cdots>v_{i_n}.
\)

In the present framework, ranking regions play two complementary roles:
\begin{enumerate}[label=\textup{(\arabic*)}]
\item they describe which rankings are compatible with the estimated score vector;
\item they provide a natural support for probability assignments once the residual noise level has been calibrated.
\end{enumerate}

Thus, the aim is not merely to compute $\widehat u$, but also to evaluate how strongly the data support the ranking region containing $\widehat u$, and how much probability mass may fall into neighboring ranking regions.

\section{Noise calibration and moderate scale variation}

\subsection{Why the residual matrix is the right object}

Once the latent ranking component and the scale-deformation component have been estimated, the residual matrix
\[
\widehat E=(\widehat E_{ij})_{1\leq i,j\leq n}
\]
defined by
\[
\widehat E_{ij}
=
x_{ij}
-
\bigl(\widehat u_i-\widehat u_j+\widehat s_i+\widehat s_j\bigr)
\]
contains the part of the observed judgments that is not explained by the fitted structured model.

This residual matrix is the natural basis for noise calibration. Indeed, the raw observed matrix $X$ contains:
\begin{itemize}
\item the latent ranking signal;
\item the systematic scale deformation;
\item and the random perturbation.
\end{itemize}
If reciprocity or triangular indicators are computed directly on $X$, these three effects are mixed together. By contrast, the same indicators computed on $\widehat E$ are intended to capture the remaining unexplained fluctuation only.

This point is conceptually important. In the present article, the structured matrix is not assumed to be additively consistent. Therefore, a triangular defect computed on the raw matrix is not in itself a sign of error. It becomes informative only after the fitted ranking and scale-deformation components have been removed.

\subsection{Residual reciprocity and residual triangular indicators}

We now introduce two quadratic indicators computed on the residual matrix.

The first one measures the remaining departure from reciprocity:
\begin{equation}\label{eq:IR2corr}
IR_2(\widehat E)
:=
\frac{2}{n(n-1)}
\sum_{1\leq i<j\leq n}
\bigl(\widehat E_{ij}+\widehat E_{ji}\bigr)^2.
\end{equation}

The second one measures the remaining triangular defect:
\begin{equation}\label{eq:IC2corr}
IC_2(\widehat E)
:=
\frac{1}{n(n-1)(n-2)}
\sum_{\substack{i,j,k\\ \text{pairwise distinct}}}
\bigl(\widehat E_{ik}-\widehat E_{ij}-\widehat E_{jk}\bigr)^2.
\end{equation}

The first indicator compares opposite residual judgments on the same pair. The second one compares the three residual judgments involved in a comparison triangle.

\begin{remark}
The quantity $IC_2(\widehat E)$ should not be interpreted as a test of whether the underlying structured matrix is additively consistent. In the present framework, the structured matrix is generally scale-deformed and therefore not additively consistent. The role of $IC_2(\widehat E)$ is different: it measures the size of the residual triangular discrepancy that remains after the fitted ranking and scale-deformation components have been extracted.
\end{remark}

\subsection{A Gaussian working model for the residual perturbation}

To interpret the indicators above, we use a Gaussian working model for the residual perturbation. This model is not intended as a literal description of human judgment. Rather, it provides a tractable approximation after the main structured effects have been removed.

We assume that the residual entries behave approximately like centered perturbations with common variance parameter $\sigma^2$. We also allow dependence between opposite residual judgments on the same pair. This dependence is summarized by a parameter $\rho$, heuristically defined by
\[
\mathrm{Cov}(\varepsilon_{ij},\varepsilon_{ji})\approx \rho\sigma^2.
\]

The parameter $\sigma$ measures the global residual noise intensity. The parameter $\rho$ measures how opposite residual judgments tend to co-vary:
\begin{itemize}
\item if $\rho=0$, opposite residual perturbations are approximately independent;
\item if $\rho>0$, they tend to move in the same direction;
\item if $\rho<0$, they tend to compensate each other;
\item if $\rho\approx -1$, the perturbation is close to an antisymmetric one.
\end{itemize}

Under this Gaussian working model, the two residual indicators satisfy the heuristic identities
\[
\mathbb E[IC_2(\widehat E)]\approx 3\sigma^2,
\qquad
\mathbb E[IR_2(\widehat E)]\approx 2(1+\rho)\sigma^2.
\]
These formulas justify the calibration procedure introduced below.

\begin{remark}
The Gaussian model is used here as a \emph{working model}. Its main purpose is to turn residual indicators into interpretable quantities. Even when the true perturbations are only approximately Gaussian, the resulting calibration remains useful as a first-order decision tool.
\end{remark}

\subsection{Estimating the effective residual noise level}

The residual triangular indicator is directly linked to the noise variance. This suggests the plug-in estimator
\begin{equation}\label{eq:sigmahatcorr}
\widehat\sigma^2:=\frac{1}{3}\,IC_2(\widehat E).
\end{equation}
Equivalently,
\[
\widehat\sigma=\sqrt{\frac{IC_2(\widehat E)}{3}}.
\]

This estimator has a transparent interpretation: it is the typical size of the unexplained perturbation left after the latent ranking and scale-deformation components have been fitted.

\begin{remark}
The quantity $\widehat\sigma$ should be viewed as an \emph{effective residual noise level}. A large value of $\widehat\sigma$ may reflect genuinely unstable judgments, but it may also indicate that the structured model remains too coarse for the data under study.
\end{remark}

\subsection{Estimating the dependence between opposite residual judgments}

The residual reciprocity indicator contains additional information on the relation between $\widehat E_{ij}$ and $\widehat E_{ji}$. Combining
\[
\mathbb E[IR_2(\widehat E)]\approx 2(1+\rho)\sigma^2
\qquad\text{and}\qquad
\mathbb E[IC_2(\widehat E)]\approx 3\sigma^2,
\]
we obtain the natural estimator
\begin{equation}\label{eq:rhohatcorr}
\widehat\rho
:=
\frac{3}{2}\frac{IR_2(\widehat E)}{IC_2(\widehat E)}-1.
\end{equation}

This quantity summarizes the residual dependence between opposite evaluations of the same pair. In a large matrix, it can be interpreted as a global cross-sectional dependence parameter.

\begin{remark}
Even when each pair is observed only once, a large comparison matrix contains many unordered pairs. Under a homogeneous perturbation model, this makes $\widehat\rho$ meaningful as a global summary of the residual dependence structure.
\end{remark}

\subsection{When is the scale deformation moderate?}

The next question is not whether the scale deformation exists, but whether it remains \emph{moderate enough} from a decision point of view.

This question has two complementary aspects.

\medskip

\noindent\textbf{A first aspect is global magnitude.}
Is the fitted scale-deformation component small relative to the fitted latent ranking component?

\medskip

\noindent\textbf{A second aspect is ranking impact.}
Even if the scale deformation is visible, does it remain too small to alter the practical reading of the ranking?

To address these two aspects, we introduce two complementary indicators.

\subsection{A global scale-deformation ratio}

Let
\[
\widehat U_{ij}:=\widehat u_i-\widehat u_j,
\qquad
\widehat S_{ij}:=\widehat s_i+\widehat s_j.
\]
We define the global scale-deformation ratio by
\begin{equation}\label{eq:Lambdacorr}
\Lambda
:=
\frac{\|\widehat S\|_F}{\|\widehat U\|_F},
\end{equation}
where $\|\cdot\|_F$ denotes the Frobenius norm.

The interpretation of $\Lambda$ is immediate:
\begin{itemize}
\item if $\Lambda$ is small, the scale deformation is globally dominated by the latent ranking component;
\item if $\Lambda$ is moderate, the deformation is present but secondary;
\item if $\Lambda$ is large, suppressing the symmetric part may remove a substantial fraction of the fitted structure.
\end{itemize}

Thus, $\Lambda$ measures the global weight of the scale deformation in the fitted structured matrix.

\subsection{A ranking-compatibility index}

A small global ratio does not always guarantee that the decision is stable. What matters for a ranking is whether the scale deformation is strong enough to compete with the smallest separation between nearby latent scores.

For this reason, we define the ranking gap of the estimated score vector $\widehat u$ by
\[
\mathrm{gap}(\widehat u)
:=
\min_{1\leq k\leq n-1}
\bigl(\widehat u_{(k)}-\widehat u_{(k+1)}\bigr),
\]
where
\[
\widehat u_{(1)}\geq \widehat u_{(2)}\geq \cdots \geq \widehat u_{(n)}
\]
denotes the decreasing rearrangement of the estimated latent scores.

We then define the scale-impact index
\begin{equation}\label{eq:Gammacorr}
\Gamma
:=
\frac{2\|\widehat s\|_\infty}{\mathrm{gap}(\widehat u)},
\end{equation}
whenever $\mathrm{gap}(\widehat u)>0$.

The factor $2$ appears because the structured deformation term has the form $\widehat s_i+\widehat s_j$.

The interpretation is as follows:
\begin{itemize}
\item if $\Gamma\ll 1$, the fitted scale deformation is too small to compete with the smallest latent ranking separation;
\item if $\Gamma$ is close to $1$, the deformation may affect the practical readability of the ranking;
\item if $\Gamma>1$, the fitted deformation is large enough to interfere with close ranking positions.
\end{itemize}

\begin{remark}
The ratio $\Lambda$ is global, whereas $\Gamma$ is local and ranking-oriented. Both viewpoints are useful in decision applications. A deformation may be globally weak and yet large enough to matter near a tie or near a threshold separating two alternatives.
\end{remark}

\subsection{A practical reading of the indicators}

The four quantities
\[
\widehat\sigma,\qquad \widehat\rho,\qquad \Lambda,\qquad \Gamma
\]
summarize the decision-relevant content of the fitted model.

\begin{itemize}
\item $\widehat\sigma$ measures the effective residual noise level.
\item $\widehat\rho$ describes the dependence pattern between opposite residual judgments.
\item $\Lambda$ measures the global importance of the fitted scale deformation.
\item $\Gamma$ measures whether this deformation is large enough to threaten ranking compatibility.
\end{itemize}

These indicators prepare the next step of the analysis. Once the residual noise has been calibrated and the scale deformation has been assessed, one can assign probabilities to ranking regions and compare this probabilistic output with the one produced by brutal reciprocal projection.

\section{Ranking-region probabilities and comparison with brutal reciprocal projection}

\subsection{From a point estimate to a probability distribution on rankings}

The vector
\[
\widehat u=(\widehat u_1,\dots,\widehat u_n)
\]
provides a central estimate of the latent ranking component. In many decision problems, however, a single strict ordering is not sufficient. When the residual noise level is non-negligible, nearby latent score vectors may induce different rankings, especially when some alternatives are close to one another.

For this reason, we now move from a point estimate to a probabilistic description of ranking uncertainty. The basic idea is to endow the latent ranking vector with a Gaussian uncertainty law centered at $\widehat u$, calibrated from the residual matrix. This makes it possible to assign probabilities to ranking regions.

\subsection{Ranking regions}

Let
\[
\mathcal H:=\left\{v=(v_1,\dots,v_n)\in\mathbb R^n:\ \sum_{i=1}^n v_i=0\right\}.
\]
This hyperplane is the natural space of centered latent ranking scores.

For each strict ordering
\[
i_1 \succ i_2 \succ \cdots \succ i_n,
\]
we define the associated ranking region by
\begin{equation}\label{eq:rankingregioncorr}
\mathcal R(i_1,\dots,i_n)
:=
\left\{v\in\mathcal H:\ v_{i_1}>v_{i_2}>\cdots>v_{i_n}\right\}.
\end{equation}
Each ranking region corresponds to one possible strict ranking of the alternatives.

Thus, once a probability distribution has been placed on the latent ranking vector, one can attach to each strict ranking the probability of the corresponding ranking region.

\begin{remark}
In concrete decision applications, it is often unnecessary to enumerate all $n!$ ranking regions. One may focus instead on the region containing $\widehat u$, on the most probable neighboring regions, or on summary events such as ``alternative $i$ belongs to the top three'' or ``alternative $i$ is ranked above alternative $j$''.
\end{remark}

\subsection{A Gaussian uncertainty law for the latent ranking vector}

We now explain how to place a Gaussian uncertainty law on the estimated latent ranking vector.

Recall that the antisymmetric part of the observed matrix satisfies
\[
K_{ij}=u_i-u_j+\eta_{ij},
\qquad
\eta_{ij}:=\frac{\varepsilon_{ij}-\varepsilon_{ji}}{2}.
\]
Under the Gaussian working model introduced in the previous section,
\[
\mathrm{Var}(\eta_{ij})=\frac{1-\rho}{2}\sigma^2.
\]
Replacing the unknown parameters by their estimators $\widehat\sigma$ and $\widehat\rho$, we obtain the following covariance estimator.

\begin{proposition}
Under the homogeneous Gaussian working model on the complete comparison graph, a natural covariance estimator for the latent ranking vector is
\begin{equation}\label{eq:Sigmahatcorr}
\widehat\Sigma_u
=
\frac{(1-\widehat\rho)\widehat\sigma^2}{2n}
\left(I-\frac1nJ\right),
\end{equation}
where $I$ is the identity matrix and $J$ is the matrix with all entries equal to $1$.
\end{proposition}

\begin{proof}
The antisymmetric residual perturbation is
\[
\eta_{ij}=\frac{\varepsilon_{ij}-\varepsilon_{ji}}{2},
\]
so
\[
\mathrm{Var}(\eta_{ij})
=
\frac14\mathrm{Var}(\varepsilon_{ij}-\varepsilon_{ji})
=
\frac{1-\rho}{2}\sigma^2.
\]
Moreover, on the complete graph,
\[
\widehat u_i-u_i=\frac1n\sum_{j=1}^n \eta_{ij}.
\]
Hence, on the centered hyperplane $\mathcal H$,
\[
\mathrm{Cov}(\widehat u-u)
=
\frac{(1-\rho)\sigma^2}{2n}
\left(I-\frac1nJ\right).
\]
Replacing $\sigma$ and $\rho$ by $\widehat\sigma$ and $\widehat\rho$ yields \eqref{eq:Sigmahatcorr}.
\end{proof}

This leads to the Gaussian approximation
\begin{equation}\label{eq:gaussianucorr}
U^\ast\sim \mathcal N_{\mathcal H}(\widehat u,\widehat\Sigma_u),
\end{equation}
that is, a Gaussian law on the centered score space $\mathcal H$ with mean $\widehat u$ and covariance matrix $\widehat\Sigma_u$.

\begin{remark}
This Gaussian law is a calibrated approximation of uncertainty on the latent ranking scores. Its purpose is not to reproduce every aspect of the elicitation process, but to translate the estimated residual noise into a probabilistic description of nearby rankings.
\end{remark}

\subsection{Probabilities of ranking regions}

Given a ranking region $\mathcal R$, we define its estimated probability by
\begin{equation}\label{eq:pregioncorr}
\widehat{\mathbb P}(\mathcal R)
:=
\mathbb P\bigl(U^\ast\in\mathcal R\bigr),
\end{equation}
where $U^\ast$ is distributed according to \eqref{eq:gaussianucorr}.

If $\mathcal R_{\widehat u}$ denotes the ranking region containing $\widehat u$, then
\[
\widehat{\mathbb P}(\mathcal R_{\widehat u})
\]
measures how strongly the data support the central strict ranking suggested by the estimated latent scores.

More generally, the family
\[
\bigl(\widehat{\mathbb P}(\mathcal R)\bigr)_{\mathcal R}
\]
provides a probabilistic map of the plausible rankings. For decision support, this is often more informative than a single ranking output.

\subsection{Monte Carlo evaluation}

For moderate or large values of $n$, the probabilities \eqref{eq:pregioncorr} are most naturally evaluated by Monte Carlo simulation.

A simple simulation procedure is as follows:
\begin{enumerate}[label=\textup{Step \arabic*.}]
\item Compute $\widehat u$, $\widehat\sigma$, and $\widehat\rho$.
\item Build the covariance matrix $\widehat\Sigma_u$ using \eqref{eq:Sigmahatcorr}.
\item Draw independent samples
\[
U^{\ast(1)},\dots,U^{\ast(N)}\sim \mathcal N_{\mathcal H}(\widehat u,\widehat\Sigma_u).
\]
\item For each sample, determine the strict ranking induced by its coordinates.
\item Estimate the probability of each ranking region by the corresponding empirical frequency.
\end{enumerate}

The same procedure can be used to evaluate summary quantities such as:
\begin{itemize}
\item the probability that one alternative is ranked first;
\item the probability that one alternative belongs to the top $k$;
\item the probability that one alternative is ranked above another.
\end{itemize}

\begin{remark}
In practice, these summary probabilities are often easier to communicate to decision makers than the full distribution over all ranking regions.
\end{remark}

\subsection{Brutal reciprocal projection}

We now compare the present structured approach with the brutal reciprocal projection. The latter starts from
\[
X^\sharp:=\frac{X-X^\top}{2}=K
\]
and completely discards the symmetric part of the observed matrix.

At first sight, one may expect the two approaches to produce different central rankings. Under the present least-squares construction, however, this is not the case.

\begin{proposition}
The latent ranking estimator obtained from the brutal reciprocal projection coincides with the estimator $\widehat u$ obtained in the structured approach.
\end{proposition}

\begin{proof}
The brutal reciprocal projection retains only the antisymmetric part $K$ of the observed matrix. The corresponding least-squares estimator is therefore exactly the minimizer of
\[
\sum_{i\neq j}\bigl(K_{ij}-(u_i-u_j)\bigr)^2
\]
under the centering constraint $\sum_i u_i=0$. This is precisely the estimator $\widehat u$ introduced earlier.
\end{proof}

Thus, the central latent ranking estimate is the same in both procedures. The main difference lies elsewhere: in the interpretation of the symmetric part and in the calibration of uncertainty.

\subsection{Where the two approaches differ}

Let
\[
\widehat U_{ij}:=\widehat u_i-\widehat u_j
\]
be the fitted latent ranking component.

Under the brutal reciprocal projection, the unexplained part of the data is
\begin{equation}\label{eq:ebrutalcorr}
\widehat E^\sharp_{ij}:=x_{ij}-\widehat U_{ij}.
\end{equation}
Under the structured approach, the unexplained part is
\begin{equation}\label{eq:estructcorr}
\widehat E_{ij}:=x_{ij}-\widehat U_{ij}-\widehat S_{ij},
\qquad
\widehat S_{ij}:=\widehat s_i+\widehat s_j.
\end{equation}
Therefore,
\begin{equation}\label{eq:ebrutalstructcorr}
\widehat E^\sharp=\widehat E+\widehat S.
\end{equation}

This identity is crucial. It shows that the brutal reciprocal projection treats the whole fitted scale deformation as unexplained variability. By contrast, the structured approach removes that component first and calibrates the residual noise only on what remains unexplained.

\subsection{Consequences for uncertainty quantification}

The practical consequence is immediate.

\begin{itemize}
\item If the fitted scale deformation is negligible, then the two approaches lead to similar uncertainty assessments.
\item If the fitted scale deformation is not negligible, then the brutal reciprocal projection tends to inflate the apparent residual variability.
\item As a result, the brutal approach may produce a more diffuse probability distribution over ranking regions.
\end{itemize}

This is precisely where the two approaches differ from a decision-support viewpoint. They may agree on the central ranking, yet disagree on how strongly that ranking is supported by the data.

\subsection{A practical reading of the comparison}

The comparison between the two approaches can be summarized as follows.

\medskip

\noindent\textbf{Same central ranking.}
Both methods produce the same estimator $\widehat u$ of the latent ranking vector.

\medskip

\noindent\textbf{Different interpretation of the symmetric part.}
The brutal method discards it immediately. The structured method interprets it as a possible scale deformation.

\medskip

\noindent\textbf{Different calibration of uncertainty.}
The brutal method may treat structured scale deformation as if it were random residual variability. The structured method separates the two effects before assigning probabilities to ranking regions.

\medskip

\noindent\textbf{Different decision output.}
The structured method provides a diagnosis of whether the symmetric part is negligible, moderate, or influential. The brutal method does not.

\subsection{When is brutal reciprocal projection acceptable?}

The indicators introduced in the previous section provide a natural answer.

\begin{itemize}
\item If $\Lambda$ is small and $\Gamma$ is small, the fitted scale deformation is weak both globally and locally. In that case, brutal reciprocal projection is likely to be harmless.
\item If $\Lambda$ is moderate but $\Gamma$ remains small, the scale deformation exists but does not substantially alter the practical reading of the ranking.
\item If $\Gamma$ is not small, then the scale deformation may interfere with close positions in the ranking, and a blind reciprocal projection is no longer satisfactory from a decision-support viewpoint.
\end{itemize}

Thus, the present approach does not reject reciprocal projection in general. Rather, it provides criteria for deciding when this simplification is harmless and when it suppresses meaningful information.

\subsection{Summary of the probabilistic output}

The structured approach produces three layers of output:
\begin{enumerate}[label=\textup{(\arabic*)}]
\item a central latent ranking estimate through $\widehat u$;
\item a calibrated uncertainty level through $\widehat\sigma$ and $\widehat\rho$;
\item a probability distribution on ranking regions, and therefore on decision-relevant ranking events.
\end{enumerate}

This output is richer than a single ranking and makes it possible to compare the structured model with brutal reciprocal projection on a probabilistic basis.

\subsection{What remains to be illustrated}

The last part of the paper should now illustrate the method numerically. This can be done in two steps:
\begin{enumerate}[label=\textup{(\arabic*)}]
\item simulated examples showing how the probabilities of ranking regions vary with the noise level and with the scale-variation component;
\item one concrete decision example comparing the structured method and the brutal reciprocal projection.
\end{enumerate}

These illustrations are important because they translate the geometric and probabilistic constructions into objects that can be read directly in decision terms.
\section{Numerical illustrations}

\subsection{Purpose of the numerical study}

The numerical study has two objectives.

First, it illustrates how the proposed method behaves under controlled simulated conditions. In particular, we want to understand how the estimated residual noise level, the scale-deformation indicators, and the ranking-region probabilities evolve when the latent ranking is perturbed by increasing noise and by increasing structured scale deformation.

Second, it compares the proposed structured approach with brutal reciprocal projection on a concrete decision example. Since both approaches produce the same central latent ranking estimate in the present model, the comparison focuses on uncertainty quantification and on the interpretation of the symmetric part of the observed matrix.

\subsection{Simulation design}

We simulate pairwise comparison matrices according to the model
\begin{equation}\label{eq:simmodel}
x_{ij}=u_i-u_j+s_i+s_j+\varepsilon_{ij},
\qquad i\neq j,
\end{equation}
with centered vectors $u$ and $s$.

The latent ranking vector $u$ is chosen so as to generate a strict ranking with controlled score gaps. A convenient choice is
\[
u_i=c\left(\frac{n+1}{2}-i\right),
\qquad 1\le i\le n,
\]
for a fixed spacing parameter $c>0$, followed by centering if needed. This yields an ordered sequence of alternatives with a simple and interpretable latent structure.

The scale-deformation vector $s$ is chosen according to different regimes:
\begin{enumerate}[label=\textup{(\alph*)}]
\item \emph{no scale deformation:} $s_i=0$ for all $i$;
\item \emph{moderate scale deformation:} $s$ is centered and satisfies
\[
2\|s\|_\infty<\mathrm{gap}(u);
\]
\item \emph{strong scale deformation:} $s$ is centered and satisfies
\[
2\|s\|_\infty\gtrsim \mathrm{gap}(u).
\]
\end{enumerate}

The residual perturbation matrix $(\varepsilon_{ij})$ is generated from a centered Gaussian model with common variance $\sigma^2$ and pairwise opposite correlation parameter $\rho$:
\[
\mathrm{Var}(\varepsilon_{ij})=\sigma^2,
\qquad
\mathrm{Cov}(\varepsilon_{ij},\varepsilon_{ji})=\rho\sigma^2.
\]
This allows us to vary separately:
\begin{itemize}
\item the \emph{noise level} through $\sigma$;
\item the \emph{dependence between opposite residual judgments} through $\rho$;
\item the \emph{importance of structured nonreciprocity} through $s$.
\end{itemize}

\subsection{Simulation scenarios}

A convenient simulation plan is to combine:
\begin{itemize}
\item several dimensions $n$ (for example $n=5,8,12$);
\item several noise levels $\sigma$ (for example low, medium, high);
\item several dependence levels $\rho$ (for example $-0.5$, $0$, $0.5$);
\item several scale-deformation regimes (none, moderate, strong).
\end{itemize}

For each parameter configuration, one generates $N_{\mathrm{rep}}$ independent matrices, applies the proposed method, and records:
\[
\widehat\sigma,\qquad
\widehat\rho,\qquad
\Lambda,\qquad
\Gamma,
\]
together with the estimated ranking-region probabilities.

The same simulations are then processed under brutal reciprocal projection in order to compare the two uncertainty assessments.

\subsection{Quantities to be reported}

For each simulation scenario, the most informative quantities are the following.

\medskip

\noindent\textbf{Residual calibration quantities.}
The empirical averages and dispersions of
\[
\widehat\sigma
\qquad\text{and}\qquad
\widehat\rho.
\]
These quantities assess whether the method correctly recovers the effective residual noise parameters.

\medskip

\noindent\textbf{Scale-deformation diagnostics.}
The empirical averages of
\[
\Lambda
\qquad\text{and}\qquad
\Gamma.
\]
These quantities indicate whether the fitted scale deformation is globally weak or potentially influential on close ranking positions.

\medskip

\noindent\textbf{Ranking uncertainty quantities.}
The probability of the central ranking region
\[
\widehat{\mathbb P}(\mathcal R_{\widehat u}),
\]
the entropy of the ranking-region distribution, and summary probabilities such as:
\begin{itemize}
\item the probability that the top-ranked alternative is correctly recovered;
\item the probability that each alternative belongs to the top $k$;
\item pairwise precedence probabilities.
\end{itemize}

\medskip

\noindent\textbf{Comparison with brutal reciprocal projection.}
The difference between the two probability distributions on ranking regions may be summarized by:
\begin{itemize}
\item the difference in central ranking-region probability;
\item the difference in top-$k$ probabilities;
\item optionally, the total variation distance between the two ranking-region distributions.
\end{itemize}

\subsection{Expected qualitative behavior}

The simulations are intended to confirm the following qualitative phenomena.

\begin{enumerate}[label=\textup{(\arabic*)}]
\item When the scale deformation is absent or negligible, the structured approach and brutal reciprocal projection should lead to similar uncertainty assessments.
\item When the scale deformation is moderate but non-negligible, the two methods should still agree on the central ranking, but the brutal reciprocal projection should tend to produce a more diffuse ranking-region distribution.
\item When the scale deformation is strong, the indicators $\Lambda$ and $\Gamma$ should detect that the symmetric part contains substantial structured information and should warn against blind reciprocal projection.
\item As the residual noise level increases, ranking-region probabilities should spread over several neighboring rankings, reflecting increasing decision uncertainty.
\end{enumerate}

These are precisely the features that the proposed framework is designed to make visible.

\subsection{Suggested presentation of the simulation results}

From a presentation point of view, the following displays are particularly useful:
\begin{itemize}
\item a table reporting the average values of $\widehat\sigma$, $\widehat\rho$, $\Lambda$, and $\Gamma$ across simulation scenarios;
\item a figure showing the probability of the central ranking region as a function of $\sigma$;
\item a figure comparing the structured method and brutal reciprocal projection in terms of top-ranked probability or top-$k$ membership probabilities;
\item one heatmap of pairwise precedence probabilities under each method.
\end{itemize}

For a decision-oriented audience, these summary displays are usually more informative than a full list of all ranking-region probabilities.

\subsection{A concrete decision example}

In addition to the simulated study, the method should be illustrated on one concrete pairwise comparison matrix arising from a realistic decision problem.

A suitable example should satisfy the following properties:
\begin{itemize}
\item the number of alternatives is moderate, so that ranking-region probabilities remain interpretable;
\item the observed matrix is visibly nonreciprocal;
\item the decision problem is simple enough to be understood without technical background.
\end{itemize}

Typical examples include:
\begin{itemize}
\item a preference elicitation exercise on a set of products or services;
\item an expert evaluation matrix comparing policy options;
\item a multicriteria synthesis matrix obtained from aggregated human judgments.
\end{itemize}

\subsection{Protocol for the real-data illustration}

For the empirical illustration, the analysis proceeds in six steps.

\begin{enumerate}[label=\textup{Step \arabic*.}]
\item Compute the antisymmetric and symmetric parts of the observed matrix.
\item Estimate the latent ranking vector $\widehat u$ and the scale-deformation vector $\widehat s$.
\item Compute the residual matrix $\widehat E$.
\item Estimate the residual quantities $\widehat\sigma$ and $\widehat\rho$, and the deformation indicators $\Lambda$ and $\Gamma$.
\item Simulate the Gaussian law on the latent ranking space and estimate ranking-region probabilities.
\item Repeat the uncertainty analysis under brutal reciprocal projection and compare the results.
\end{enumerate}

The empirical discussion should then focus on three questions:
\begin{itemize}
\item Is the fitted scale deformation negligible, moderate, or influential?
\item How concentrated is the probability distribution on ranking regions?
\item Does brutal reciprocal projection materially alter the uncertainty assessment?
\end{itemize}

\subsection{Decision-oriented reading of the empirical results}

The purpose of the real-data illustration is not merely to show that the method runs numerically. It is to demonstrate how the method supports a more nuanced reading of pairwise comparison data.

In particular, the empirical example should make it possible to distinguish between the following situations:
\begin{itemize}
\item \emph{stable decision setting:} low residual noise, weak scale deformation, concentrated ranking-region probabilities;
\item \emph{stable ranking but non-negligible structured asymmetry:} same central ranking, but visible scale deformation and different uncertainty calibration;
\item \emph{fragile decision setting:} high residual noise and diffuse ranking-region probabilities.
\end{itemize}

This distinction is central for decision support. A ranking may be numerically available in all three cases, but its interpretation is not the same.
\section{Implemented numerical examples}
\subsection{A worked example and a local Monte Carlo study}

We now illustrate the proposed methodology on a simple four-alternative example. The purpose of this subsection is twofold. First, we exhibit one explicit noisy matrix for which brutal reciprocal projection becomes non-admissible for strict ranking, while least-squares consistencization leads to a ranking different from the original latent one. Second, we show by a local Monte Carlo study that this phenomenon is not anecdotal.

\subsubsection*{Step 1. A reciprocal and additively consistent latent matrix}

Consider the latent score vector
\[
u=(0.15,\;0.05,\;0,\;-0.20),
\]
which induces the strict latent ranking
\[
1\succ 2\succ 3\succ 4.
\]
The associated reciprocal and additively consistent comparison matrix is
\[
A=(a_{ij}),\qquad a_{ij}=u_i-u_j,
\]
that is,
\[
A=
\begin{pmatrix}
0 & 0.10 & 0.15 & 0.35\\
-0.10 & 0 & 0.05 & 0.25\\
-0.15 & -0.05 & 0 & 0.20\\
-0.35 & -0.25 & -0.20 & 0
\end{pmatrix}.
\]

\subsubsection*{Step 2. A noisy nonreciprocal observed matrix}

Consider the observed matrix
\[
X=
\begin{pmatrix}
0 & -0.07 & 0.03 & 0.08\\
0.13 & 0 & -0.10 & 0.06\\
-0.01 & 0.06 & 0 & 0.01\\
-0.12 & -0.04 & -0.03 & 0
\end{pmatrix}.
\]
This matrix is nonreciprocal. For instance,
\[
x_{12}=-0.07\neq -x_{21}=-0.13.
\]
It can be written as
\[
X=A+B
\]
with
\[
B=
\begin{pmatrix}
0 & -0.17 & -0.12 & -0.27\\
0.23 & 0 & -0.15 & -0.19\\
0.14 & 0.11 & 0 & -0.19\\
0.23 & 0.21 & 0.17 & 0
\end{pmatrix}.
\]

Thus, a reciprocal and additively consistent latent matrix has been transformed into a noisy nonreciprocal observed matrix.

\subsubsection*{Step 3. Brutal reciprocal projection}

The brutal reciprocal projection is
\[
X^\sharp:=\frac{X-X^\top}{2}.
\]
A direct computation gives
\[
X^\sharp=
\begin{pmatrix}
0 & -0.10 & 0.02 & 0.10\\
0.10 & 0 & -0.08 & 0.05\\
-0.02 & 0.08 & 0 & 0.02\\
-0.10 & -0.05 & -0.02 & 0
\end{pmatrix}.
\]

\subsubsection*{Step 4. Non-admissibility for strict ranking}

The sign pattern of $X^\sharp$ is not compatible with any strict ranking. Indeed,
\[
x^\sharp_{12}<0,\qquad x^\sharp_{23}<0,\qquad x^\sharp_{13}>0,
\]
so that
\[
2\succ 1,\qquad 3\succ 2,\qquad 1\succ 3.
\]
This produces a cycle, hence no strict global ranking can be compatible with the signs of $X^\sharp$.

Therefore, brutal reciprocal projection produces here a reciprocal matrix that is \emph{non-admissible} for strict ranking.

\subsubsection*{Step 5. Least-squares consistencization}

We now project $X^\sharp$ onto the reciprocal and additively consistent matrices of the form
\[
\widehat A_{ij}=\widehat u_i-\widehat u_j.
\]
On the complete graph, the least-squares estimator is given by the row-average formula
\[
\widehat u_i=\frac{1}{4}\sum_{j=1}^4 x^\sharp_{ij}.
\]
The row sums are
\[
0.02,\qquad 0.07,\qquad 0.08,\qquad -0.17,
\]
hence
\[
\widehat u=(0.005,\;0.0175,\;0.020,\;-0.0425).
\]
The recovered ranking is therefore
\[
3\succ 2\succ 1\succ 4,
\]
which differs from the original latent ranking
\[
1\succ 2\succ 3\succ 4.
\]

This worked example shows that a noisy nonreciprocal perturbation of a perfectly reciprocal and consistent latent matrix may lead, after brutal reciprocal projection, to a non-admissible sign pattern, and after least-squares consistencization, to a ranking different from the original latent one.

\subsubsection*{Step 6. A local Monte Carlo study}

The previous example is informative, but by itself it remains anecdotal. We therefore complement it with a local Monte Carlo study around the same latent matrix $A$.

For each simulation run, we generate
\[
X=A+B,
\]
where the diagonal entries are fixed to zero and the off-diagonal perturbations are independent Gaussian random variables:
\[
b_{ij}\sim \mathcal N(0,\sigma^2),
\qquad i\neq j.
\]
For each realization, we compute:
\begin{enumerate}
\item the brutal reciprocal projection
\[
X^\sharp=\frac{X-X^\top}{2};
\]
\item whether the sign pattern of $X^\sharp$ is admissible for a strict ranking;
\item the least-squares consistent approximation of $X^\sharp$;
\item whether the resulting ranking coincides with the original latent ranking
\[
1\succ 2\succ 3\succ 4.
\]
\end{enumerate}

For each value of $\sigma$, we estimate the following probabilities:
\begin{itemize}
\item $p_{\mathrm{NA}}$: probability that brutal reciprocal projection is non-admissible for strict ranking;
\item $p_{\mathrm{WR}}$: probability that least-squares consistencization yields a ranking different from the original latent one;
\item $p_{\mathrm{Both}}$: probability that both events occur simultaneously.
\end{itemize}

Using $10^5$ Monte Carlo replications for each value of $\sigma$, we obtain the following estimates:

\begin{center}
\begin{tabular}{c|ccc}
\hline
$\sigma$ & $p_{\mathrm{NA}}$ & $p_{\mathrm{WR}}$ & $p_{\mathrm{Both}}$\\
\hline
$0.05$ & $0.00015$ & $0.02286$ & $0.00000$\\
$0.10$ & $0.03245$ & $0.18163$ & $0.00889$\\
$0.15$ & $0.11135$ & $0.34452$ & $0.05062$\\
$0.20$ & $0.19690$ & $0.47416$ & $0.11029$\\
\hline
\end{tabular}
\end{center}

These numerical results are consistent with the conceptual picture developed in the paper.

\begin{itemize}
\item For small noise levels, non-admissibility is rare, but ranking errors after consistencization already occur with non-negligible probability.
\item As the noise level increases, brutal reciprocal projection becomes more frequently non-admissible.
\item At the same time, the least-squares consistencization increasingly recovers a ranking different from the original latent one.
\end{itemize}

Hence the phenomenon exhibited by the explicit worked example is not exceptional. It appears with non-negligible frequency as soon as the noise level becomes comparable with the smallest latent ranking gaps.

\subsubsection*{Step 7. Interpretation}

This subsection supports three conclusions.

First, even when the latent matrix is reciprocal and additively consistent, a noisy observation may become nonreciprocal in a way that materially affects the downstream ranking analysis.

Second, brutal reciprocal projection may produce a reciprocal matrix whose sign structure is incompatible with any strict ranking. In that case, a further consistencization step becomes unavoidable.

Third, this consistencization step may stabilize the matrix algebraically while recovering a ranking different from the original latent one. In other words, reciprocity restoration and least-squares consistencization may repair the matrix formally without preserving the original decision content.

For the present article, this example therefore plays the role of a baseline warning case: even before introducing structured scale deformation, brute-force reciprocal correction may already alter ranking admissibility and ranking recovery under realistic noisy perturbations.

\subsection{Fixed noise and increasing scale deformation}

The previous subsection showed that noise alone may already create serious ranking distortions after brutal reciprocal projection and least-squares consistencization. We now turn to the second phenomenon studied in this paper: the effect of a \emph{structured scale deformation} when the residual noise level is kept fixed.

The purpose of this subsection is different. Here the central latent ranking is intentionally kept unchanged, while the scale-deformation component is increased progressively. The goal is to show that, even when the point ranking estimate remains the same, the uncertainty calibration produced by brutal reciprocal projection may become substantially distorted.

\subsubsection*{Step 1. A latent ranking with a fixed gap structure}

Consider the latent ranking vector
\[
u=(0.30,\;0.12,\;-0.06,\;-0.36),
\]
which induces the strict ranking
\[
1\succ 2\succ 3\succ 4.
\]
The consecutive score gaps are
\[
0.18,\qquad 0.18,\qquad 0.30,
\]
so that
\[
\mathrm{gap}(u)=0.18.
\]

The associated reciprocal and additively consistent comparison matrix is
\[
A=(a_{ij}),\qquad a_{ij}=u_i-u_j.
\]

\subsubsection*{Step 2. A one-parameter family of scale deformations}

We now introduce a centered deformation profile
\[
q=(0.08,\;0.03,\;-0.03,\;-0.08),
\qquad \sum_{i=1}^4 q_i=0,
\]
and define
\[
s^{(\tau)}:=\tau q,
\qquad \tau\ge 0.
\]
The associated scale-deformed structured matrix is
\[
M^{(\tau)}_{ij}:=u_i-u_j+s^{(\tau)}_i+s^{(\tau)}_j.
\]

Since
\[
2\|q\|_\infty=0.16<\mathrm{gap}(u)=0.18,
\]
the condition
\[
2\|s^{(\tau)}\|_\infty<\mathrm{gap}(u)
\]
holds for all
\[
0\le \tau<\tau_c:=\frac{0.18}{0.16}=1.125.
\]
Hence, for all $\tau<1.125$, the scale-deformed structured matrix remains ranking-compatible with the latent ranking vector $u$.

\subsubsection*{Step 3. A fixed reasonable noise level}

We now fix the residual perturbation level once and for all. For all $i\neq j$, we assume
\[
\varepsilon_{ij}\sim \mathcal N(0,\sigma_0^2),
\qquad \sigma_0=0.10,
\]
with independence across ordered pairs and therefore
\[
\rho_0=0.
\]

The observed matrix is then
\[
X^{(\tau)}_{ij}=u_i-u_j+s^{(\tau)}_i+s^{(\tau)}_j+\varepsilon_{ij}.
\]

This setup reflects the intended decision interpretation: the random perturbation level is kept fixed, while the systematic deformation of the evaluation scale becomes the sole varying parameter.

\subsubsection*{Step 4. What the structured method should recover}

Under the structured approach developed in this paper, the antisymmetric part is used to estimate the latent ranking vector, while the symmetric part is used to estimate the scale deformation. Since the antisymmetric part does not depend on $s^{(\tau)}$, the central ranking estimate remains centered around the same latent vector $u$ for all values of $\tau$.

Moreover, once the fitted scale deformation has been removed, the residual perturbation remains governed by the same noise level $\sigma_0$ and by the same dependence parameter $\rho_0=0$. Therefore, under the idealized model, the structured method should recover approximately
\[
\widehat\sigma\approx 0.10,
\qquad
\widehat\rho\approx 0,
\]
uniformly in $\tau$.

This is the main benchmark against which brutal reciprocal projection should be compared.

\subsubsection*{Step 5. What brutal reciprocal projection does instead}

The brutal reciprocal projection discards the symmetric part of the observed matrix at once. In the present model, this means that the whole scale-deformation component is treated as if it were unexplained variability.

Let
\[
S^{(\tau)}_{ij}:=s^{(\tau)}_i+s^{(\tau)}_j.
\]
Then, under brutal reciprocal projection, the unexplained part is effectively
\[
E^{\sharp(\tau)}=E^{(\tau)}+S^{(\tau)},
\]
where $E^{(\tau)}$ denotes the residual perturbation that would remain after correct structured fitting.

Because $\sum_i q_i=0$, one has
\[
\|q\|_2^2
=
0.08^2+0.03^2+(-0.03)^2+(-0.08)^2
=
0.0146.
\]
Hence
\[
\|s^{(\tau)}\|_2^2=0.0146\,\tau^2.
\]

Using the formulas established earlier in the paper, the expected brutal indicators are
\begin{equation}\label{eq:IC2brutal_tau}
\mathbb E\!\left[IC_2\bigl(E^{\sharp(\tau)}\bigr)\right]
=
3\sigma_0^2+\frac{4}{4}\|s^{(\tau)}\|_2^2
=
0.03+0.0146\,\tau^2,
\end{equation}
and
\begin{equation}\label{eq:IR2brutal_tau}
\mathbb E\!\left[IR_2\bigl(E^{\sharp(\tau)}\bigr)\right]
=
2(1+\rho_0)\sigma_0^2+\frac{8(4-2)}{4(4-1)}\|s^{(\tau)}\|_2^2
=
0.02+0.01947\,\tau^2.
\end{equation}

Therefore, the brutal method interprets the scale deformation as an increase of residual noise and as a spurious positive dependence between opposite residual judgments.

\subsubsection*{Step 6. Expected calibration under both methods}

The structured method should keep recovering the true residual calibration
\[
\widehat\sigma\approx 0.10,
\qquad
\widehat\rho\approx 0,
\]
whereas the brutal method yields the apparent estimators
\[
\widehat\sigma_\sharp^2
\approx
\frac13\Bigl(0.03+0.0146\,\tau^2\Bigr)
=
0.01+0.004867\,\tau^2,
\]
and
\[
\widehat\rho_\sharp
\approx
\frac32\,
\frac{0.02+0.01947\,\tau^2}{0.03+0.0146\,\tau^2}
-1.
\]

At the same time, the two scale-deformation diagnostics introduced earlier are
\[
\Lambda(\tau)=\frac{\|S^{(\tau)}\|_F}{\|U\|_F},
\qquad
\Gamma(\tau)=\frac{2\|s^{(\tau)}\|_\infty}{\mathrm{gap}(u)}.
\]
A direct computation gives
\[
\Lambda(\tau)\approx 0.248\,\tau,
\qquad
\Gamma(\tau)\approx 0.889\,\tau.
\]

The following table summarizes the expected behavior for three representative values of the scale-deformation parameter.

\begin{center}
\begin{tabular}{c|cc|cc|cc}
\hline
$\tau$ & $\Lambda(\tau)$ & $\Gamma(\tau)$
& $\widehat\sigma_{\mathrm{struct}}$
& $\widehat\sigma_{\sharp}$
& $\widehat\rho_{\mathrm{struct}}$
& $\widehat\rho_{\sharp}$\\
\hline
$0$   & $0.000$ & $0.000$ & $0.100$ & $0.100$ & $0.000$ & $0.000$\\
$0.5$ & $0.124$ & $0.444$ & $0.100$ & $0.106$ & $0.000$ & $0.109$\\
$1.0$ & $0.248$ & $0.889$ & $0.100$ & $0.122$ & $0.000$ & $0.327$\\
\hline
\end{tabular}
\end{center}

\subsubsection*{Step 7. Interpretation}

This parameter study shows a phenomenon that is central for the present article.

\begin{enumerate}
\item The latent ranking itself is unchanged: the parameter $\tau$ affects only the symmetric structured component.
\item As long as $\tau<1.125$, the structured matrix remains ranking-compatible with the original latent ranking.
\item Under the structured method, the calibrated residual noise level remains essentially constant, because the scale deformation is fitted and removed before noise calibration.
\item Under brutal reciprocal projection, the same scale deformation is misread as residual variability. As $\tau$ increases, the apparent noise level is inflated and the apparent dependence parameter becomes positive, even though the true residual perturbation has $\rho_0=0$.
\end{enumerate}

Thus, the main issue is not a change in the central ranking estimate. It is a change in the \emph{interpretation of uncertainty}. The structured method recognizes the symmetric part as a deformation of the evaluation scale, whereas brutal reciprocal projection absorbs it into noise.

\subsubsection*{Step 8. Decision-oriented conclusion of the example}

From a decision-support viewpoint, this example illustrates the main message of the paper in its cleanest form.

When the scale deformation is weak, both methods lead to nearly identical conclusions. When the scale deformation becomes moderate, the central ranking may still be stable, but the brutal method starts producing an artificially more diffuse uncertainty assessment. When the scale deformation approaches the ranking-compatibility threshold, this discrepancy becomes substantial.

In other words, two procedures may agree on the same central ranking and yet disagree markedly on how strongly the data support that ranking. This is precisely why a structured treatment of nonreciprocity is preferable to an immediate reciprocal correction when the symmetric part may carry meaningful decision information.
\subsection{Conclusion of the numerical illustrations}

The numerical illustrations support two complementary messages.

First, the worked example with a reciprocal and additively consistent latent matrix shows that noise alone may already create serious downstream distortions. A nonreciprocal perturbation may lead, after brutal reciprocal projection, to a reciprocal matrix whose sign pattern is not admissible for any strict ranking. Moreover, the subsequent least-squares consistencization may recover a ranking different from the original latent one. Thus, even in the absence of structured scale deformation, brute-force reciprocity restoration may alter the decision content of the data.

Second, the parameter study with fixed noise and increasing scale deformation highlights a different phenomenon. In that setting, the central latent ranking remains unchanged, and the structured matrix stays ranking-compatible over a substantial range of deformation levels. However, brutal reciprocal projection progressively misinterprets the symmetric structured component as residual randomness. As a consequence, it inflates the apparent noise level and produces a more diffuse uncertainty assessment on the ranking regions, although the true residual perturbation has not changed.

Taken together, these two examples clarify the distinction at the heart of the paper. Noise may damage ranking recovery directly, but structured nonreciprocity has a different effect: it primarily changes the way uncertainty should be calibrated and interpreted. This is why the proposed methodology does not start by imposing reciprocity blindly. Instead, it separates latent ranking information, scale deformation, and residual perturbation before assigning probabilities to ranking regions.

From a decision-oriented viewpoint, the practical lesson is clear. A reciprocal correction may be harmless when the symmetric part is negligible, but it may become misleading as soon as that part contains a moderate but meaningful structured effect. In such situations, the central ranking alone is not enough: what matters is also how strongly that ranking is supported, how diffuse the neighboring ranking regions are, and whether the apparent uncertainty comes from genuine noise or from an interpretable deformation of the evaluation scale.
\section{Conclusion}

This article proposes a structured approach to noisy nonreciprocal pairwise comparison matrices. The starting point is a simple but important observation: nonreciprocity should not automatically be interpreted as a defect to be removed. Part of it may reflect a systematic variation of the evaluation scale, while another part may be due to random perturbation.

To capture this distinction, we introduced a model in which the observed matrix is decomposed into three components:
\begin{itemize}
\item a latent ranking component;
\item a structured scale-deformation component;
\item a residual perturbation component.
\end{itemize}
This decomposition leads to explicit estimators of the latent ranking scores and of the scale deformation, and then to a residual calibration of the effective noise level.

The resulting framework provides four kinds of information that are directly relevant for decision analysis:
\begin{enumerate}[label=\textup{(\arabic*)}]
\item a central latent ranking estimate;
\item an estimate of the effective residual noise level;
\item indicators measuring the global and local impact of the scale deformation;
\item a probability distribution on ranking regions.
\end{enumerate}

One of the main conclusions is that the structured approach and brutal reciprocal projection do not differ primarily at the level of the central ranking estimate, which is the same in the present least-squares setting. Their main difference lies in the interpretation of the symmetric part and therefore in the calibration of uncertainty. When the fitted scale deformation is non-negligible, brutal reciprocal projection may treat structured asymmetry as if it were random noise, thereby producing a more diffuse and potentially misleading uncertainty assessment.

From a methodological viewpoint, the paper therefore suggests a change of perspective. The relevant question is not only whether a nonreciprocal matrix should be projected onto the reciprocal ones, but first whether the observed asymmetry is negligible, moderate, or influential from a decision point of view.

Several extensions are possible. First, the Gaussian working model could be replaced or complemented by heavier-tailed or heteroscedastic perturbation models. Second, incomplete comparison matrices and sparse comparison graphs should be treated explicitly. Third, the present framework could be extended to repeated judgments or panel data, where the distinction between scale deformation and random perturbation may become even more informative. Finally, the ranking-region approach could be connected with robustness analysis and sensitivity analysis in multicriteria decision-making.

More broadly, the proposed method aims to enrich the interpretation of pairwise comparison data. Rather than forcing an immediate reciprocal correction, it offers a way to distinguish latent ranking information, structured asymmetry, and residual uncertainty. In this sense, it provides not only a ranking, but also a diagnosis of how that ranking should be trusted and interpreted.

\vskip 12pt

\paragraph{\bf Acknowledgements:} J.-P.M thanks the France 2030 framework programme Centre Henri Lebesgue ANR-11-LABX-0020-01 
for creating an attractive mathematical environment.

\vskip 12pt

\paragraph{\bf Author's Note on AI Assistance.}
Portions of the text were developed with the assistance of a generative language model (OpenAI ChatGPT). The AI was used to assist with drafting, editing, and standardizing the bibliography format. All mathematical content, structure, and theoretical constructions were provided, verified, and curated by the author. The author assumes full responsibility for the correctness, originality, and scholarly integrity of the final manuscript.

\end{document}